\documentclass[10pt,a4paper]{article}
\usepackage{a4wide}

\usepackage{graphicx}
\usepackage{amsmath,amssymb}
\usepackage{amsthm}
\usepackage{cancel}
\usepackage{hyperref}
\usepackage{longtable}
\usepackage{cases}
\usepackage[table]{xcolor}
\usepackage{thm-restate}
\allowdisplaybreaks

\usepackage[ruled,noend]{algorithm2e}

\usepackage{tikz}
\usetikzlibrary{arrows.meta}

\usepackage{AM-style}

\theoremstyle{plain}

\newtheorem{corollary}{Corollary}
\newtheorem{proposition}{Proposition}

\theoremstyle{definition}
\newtheorem{definition}{Definition}

\theoremstyle{remark}
\newtheorem{example}{Example}

\usepackage{authblk}

\title{Contrastive Explanations for Ar\-gu\-men\-tation-Based Conclusions}
\author[1]{AnneMarie Borg}
\author[1,2]{Floris Bex}
\affil[1]{Department of Information and Computing Sciences, Utrecht University}
\affil[2]{Tilburg Institute for Law, Technology, and Society, Tilburg University}
\date{}

\makeatletter
\def\blfootnote{\xdef\@thefnmark{}\@footnotetext}
\makeatother

\begin{document}
\maketitle

\begin{abstract}
In this paper we discuss \emph{contrastive explanations} for formal argumentation -- the question why a certain argument (the fact) can be accepted, whilst another argument (the foil) cannot be accepted under various extension-based semantics. The recent work on explanations for argumentation-based conclusions has mostly focused on providing minimal explanations for the (non-)acceptance of arguments. What is still lacking, however, is a proper argumentation-based interpretation of contrastive explanations. We show under which conditions contrastive explanations in abstract and structured argumentation are meaningful, and how argumentation allows us to make implicit foils explicit.
\end{abstract}

\section{Introduction}

\emph{Explainable AI} (XAI) has become an important research direction in AI~\cite{SamekMVHM19towardsXAI}. AI systems are being applied in a variety of real-life situations in different domains and with different users. It is therefore essential that such systems are able to give explanations that provide insight into the underlying decision models and techniques, so that users can understand, trust and validate the system, and experts can verify that the system works as intended. Most of the research in XAI is directed at explaining decisions of subsymbolic machine learning algorithms~(cf.~\cite{SamekWM17explainable}), but explanations also play an important role in clarifying the decisions of symbolic algorithms~\cite{LacaveD04review}, particularly as such algorithms are all-pervasive in everyday systems. \blfootnote{Full version of an extended abstract forthcoming in Proceedings of the 21st International Conference
on Autonomous Agents and Multiagent Systems (AAMAS 2022).}

One area in symbolic AI that has seen a number of real-world applications is formal argumentation~\cite{Atkinson17applications}. Two central concepts in formal argumentation are \emph{abstract argumentation frameworks} \cite{Dung95}, sets of arguments and the attack relations between them, and \emph{structured} or \emph{logical argumentation frameworks}~\cite{BesnardGHMPS14intro}, where arguments are constructed from a knowledge base and a set of rules, and the attack relation is based on the individual elements in the arguments. Common for argumentation frameworks, abstract and structured, is that we can determine their extensions, sets of arguments that can collectively be considered as acceptable, under different semantics~\cite{Dung95}. The combination of an argumentation framework and its extensions is a \emph{global} explanation: what can we conclude from the model as a whole? However, often we would prefer simpler, more compact explanations for the acceptability of an individual argument, a \emph{local} explanation for a particular decision or conclusion~\cite{EdwardsV17slave}. A number of methods for determining local explanations for the (non-)ac\-cep\-ta\-bility of arguments have been proposed~\cite{BorgB20basic,FanT15accept,FanT15nonaccept,GarciaCRS13dialecticalexpl,LiaovdT20,SaribaturWW20}. What is still lacking, however, is an argumentation-based interpretation of \emph{contrastive explanations}.

Contrastiveness is central to local explanations \cite{Lipton90,Miller18,Miller19}: when people ask \emph{`Why P?'}, they often mean \emph{`Why P rather than Q?'} -- here $P$ is called the \emph{fact} and $Q$ is called the \emph{foil}~\cite{Lipton90}. The answer to the question is then to explain as many of the differences between fact and foil as possible. Like for XAI in general, much of the research on contrastive explanations is done in the context of machine learning (e.g. \cite{DhurandharCLTTSD18,StepinACP21contrastive,vdWaaRDBN18}). In the literature on formal argumentation, there has been no such work, the existing work focusing on \emph{`Why is argument A (not) acceptable?'} instead of the contrastive question \emph{`Why is argument A acceptable and argument B not?'} (or vice versa). While there are other forms of contrastive questions, we choose this one since it is intuitive, allows for a variety of foils and it can be interesting for both expert and lay users of an application. 

To study contrastive explanations for argumentation-based conclusions, we extend the basic framework from~\cite{BorgB20basic}. With that framework, explanations for accepted and non-accepted arguments or formulas can be formulated in a variety of ways.  The main idea of the introduced contrastive explanations is that these return the common elements of the basic acceptance explanation of the fact and the basic non-acceptance explanation of the foil (or vice versa). We show that in almost all situations these explanations are meaningful, i.e., that such common elements exist. Additionally we show that, due to the explicit notion of conflict within argumentation, we can provide contrastive explanations when the foil is not explicitly known. This is an advantage of formal argumentation, since determining the foil is a challenge for an AI system.

The paper is structured as follows: we briefly discuss some directly related work and provide an example of a real-life application that benefits from our contrastive explanations. We recall abstract argumentation as introduced in~\cite{Dung95} in Section~\ref{sec:Preliminaries}. Then, in Section~\ref{sec:Explanations}, the framework from~\cite{BorgB20basic} is recalled and some new results for acceptance and non-acceptance explanations are shown. In Section~\ref{sec:Contrastive} contrastive explanations are introduced and it is shown how, in formal argumentation, the foil can be determined when it is not explicitly stated. In Section~\ref{sec:Structured} the introduced contrastive explanations are applied to ASPIC$^+$~\cite{Prakken10} and we conclude in Section~\ref{sec:Conclusion}.

\section{Related Work}
\label{sec:RelatedWork}

XAI has been investigated in many directions, for a variety of approaches to AI, including formal argumentation. 
As mentioned in the introduction, we are interested in \emph{contrastive local explanations for conclusions derived from formal argumentation}, where the idea is that the proposed method can be applied to any Dung-style argumentation framework to generate contrastive explanations. While contrastive explanations for learning-based decisions have been investigated extensively (see~\cite{StepinACP21contrastive} for a recent overview), there are no results on contrastive explanations for argumentation-based conclusions. 

Some research on local explanations for argumentation-based conclusions already exists. For example,~\cite{GarciaCRS13dialecticalexpl} introduce explanations for claims as triples of sets of dialectical trees for abstract argumentation and DeLP and Fan and Toni introduced explanations as dispute trees for accepted arguments in abstract argumentation and ABA in~\cite{FanT15accept} and for non-accepted arguments in abstract argumentation in~\cite{FanT15nonaccept}. Even more recently, explanation semantics, where accepted arguments are labeled with sets of explanation arguments, were introduced in~\cite{LiaovdT20} and explanations for non-accepted arguments as minimal subframeworks are studied in~\cite{NiskanenJ20smallexpl,SaribaturWW20}.  

For this paper, we take the framework from~\cite{BorgB20basic}, as it is the only one that allows for acceptance and non-acceptance explanations in terms of sets of arguments. While acceptance and non-acceptance is necessary when defining contrastive explanations (see Section~\ref{sec:Contrastive}), explanations in terms of sets of arguments make it easier to process the explanations. Additionally, unlike the other frameworks, the explanations from~\cite{BorgB20basic} make it possible to present explanations derived from a structured setting in terms of elements of arguments (e.g., premises or rules), rather than full arguments. Therefore, to the best of our knowledge, this is the first research on \emph{contrastive} local explanations for conclusions derived from either abstract or structured formal argumentation. 

\section{Example Scenario}
\label{sec:Applying}

At the Netherlands Police several argumentation-based applications have been implemented~\cite{BexTP16iac}. These applications are aimed at assisting the police at working through high volume tasks, leaving more time for tasks that require human attention.  
As for any AI application, these applications should be able to provide an explanation for the derived decision. In this paper we will illustrate how the proposed contrastive explanations can be applied in an application that helps the police to identify malafide webshops~\cite{OdekerkenB20webshop}.\footnote{We work with an adjusted version of the actual system to make the conclusions more interesting from an argumentative perspective and since the system cannot be published.} 

Suppose that a complaint has been filed ($\textit{cf}$) about a webshop, that it is malafide ($m$). Usually, when a complaint is filed, an investigation into the webshop is done ($\textit{iw}$, rule $d_1$) and, when an investigation is done and it is found that the url of the webshop is suspicious ($\textit{sa}$), the webshop is found to be malafide (rule $d_3$). Now, a complaint can be retracted ($\textit{rc}$) in which case $d_1$ is not applicable (rule $d_2$), unless the owner of the webshop is known by the police ($\textit{kp}$, rule $d_5$). Similarly, if the address is registered at the chamber of commerce ($\textit{ka}$) then rule $d_3$ is not applicable (rule $d_4$), unless the registration was recently retracted ($\textit{rr}$, rule $d_6$). 

We formalize this scenario by creating the following arguments: 
\begin{itemize}
    \item[$A_1$] \emph{cf}: a complaint has been filed;
    \item[$A_2$] \emph{rc}: the complaint has been retracted;
    \item[$A_3$] \emph{sa}: the url of the webshop is suspicious;
    \item[$A_4$] \emph{ka}: the url is registered at the chamber of commerce;
    \item[$A_5$] \emph{kp}: the owner of the webshop is known by the police;
    \item[$A_6$] \emph{rr}: the registration was recently retracted;
    \item[$B_1$] \emph{iw}: an investigation into the webshop is done;
    \item[$B_2$] $\neg n(d_1)$: the rule that, when a complaint is filed an investigation into the webshop is done, is not applicable;
    \item[$B_3$] $\neg\textit{rc}$: the complaint cannot be retracted;
    \item[$B_4$] \emph{m}: the webshop is malafide;
    \item[$B_5$] $\neg n(d_3)$: the rule that if an invesitation into the webshp is done and the url of the webshop is suspicious, then the webshop is malafide, is not applicable;
    \item[$B_6$] $\neg\textit{ka}$: the url is not registered at the chamber of commerce. 
\end{itemize}
These arguments give rise to the following conflicts: $A_2$ and $B_3$ are in conflict with each other and, similarly, $A_4$ and $B_6$. $B_3$ [resp.\ $B_6$] causes a conflict with $B_2$ [resp.\ $B_5$] since \emph{rc} [resp.\ \emph{ka}] was used in the construction of $B_2$ [resp.\ $B_5$]. Finally, $B_2$ [resp.\ $B_5$] causes a conflict with $B_1$ and $B_4$ [resp.\ $B_4$] since $B_4$ (and $B_1$) is constructed with the use of rules $d_1$ and $d_3$. The resulting graphical representation can be found in Figure~\ref{fig:Webshop}, where the nodes represent the arguments and the arrows represent the attacks.

\begin{figure}[ht]
  \centering
  \begin{tikzpicture}
    \node[draw,circle] (A1) at (0,4) {$A_1$};
    \node[draw,circle] (A2) at (2,4) {$A_2$};
    \node[draw,circle] (A3) at (4,4) {$A_3$};
    \node[draw,circle] (A4) at (6,4) {$A_4$};
    \node[draw,circle] (A5) at (8,4) {$A_5$};
    \node[draw,circle] (A6) at (8,2.5) {$A_6$};
    
    \node[draw,circle] (B1) at (0,1) {$B_1$};
    \node[draw,circle] (B2) at (2,1) {$B_2$};
    \node[draw,circle] (B3) at (2,2.5) {$B_3$};
    \node[draw,circle] (B4) at (4,1) {$B_4$};
    \node[draw,circle] (B5) at (6,1) {$B_5$};
    \node[draw,circle] (B6) at (6,2.5) {$B_6$};
    
    \draw[->] (B2) -- (B1); 
    \draw[->] (B2) -- (B4);
    \draw[->] (B5) -- (B4);
    \draw[->] (B3) [bend left] to (A2);
    \draw[->] (A2) [bend left] to (B3);
    \draw[->] (B6) [bend left] to (A4);
    \draw[->] (A4) [bend left] to (B6);
    \draw[->] (B3) -- (B2);
    \draw[->] (B6) -- (B5);
  \end{tikzpicture}
  \caption{Graphical representation of the AF $\calAF_1$.}
  \label{fig:Webshop}
\end{figure}
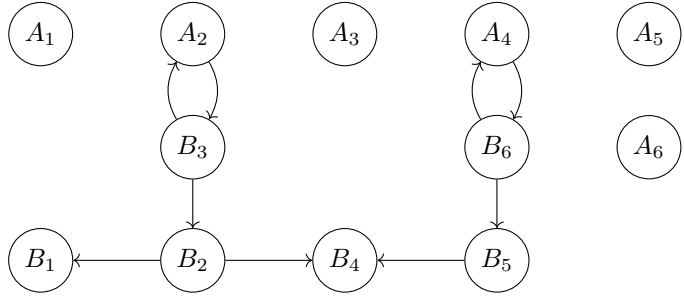

In this scenario, the webshop about which a complaint is filed is malafide if the complaint is not retracted (or the owner of the webshop is known by the police) and the url of the webshop is suspicious and not currently registered at the chamber of commerce. As a result, there are several criteria which can make a webshop malafide, each with their own exceptions. There might therefore be a variety of reasons for a given conclusion. In this paper we will show how explanations can be tailored to a specific reason. 

\section{Preliminaries}
\label{sec:Preliminaries}

We focus on explanations for conclusions derived from Dung-style argumentation frameworks. This section is very compact, see, e.g.,~\cite{Dung95} for a more gentle introduction. 

An \emph{abstract argumentation framework\/} (AF)~\cite{Dung95} is a pair $\calAF = \tuple{\Args, \attack}$, where $\Args$ is a set of \emph{arguments\/} and $\attack\subseteq\Args\times\Args$ is an \emph{attack relation\/} on these arguments. An argumentation framework can be viewed as a directed graph, in which the nodes represent arguments and the arrows represent the attacks, see also Figure~\ref{fig:Webshop}. 

\begin{example}
  \label{ex:abstractAF}
  Figure~\ref{fig:Webshop} represents the argumentation framework $\calAF_1 = \tuple{\Args_1,\attack_1}$ where $\Args_1 = \{A_1,A_2,A_3,A_4,A_5,A_6,B_1,\allowbreak B_2,\allowbreak B_3,\allowbreak B_4,B_5,B_6\}$ and $\attack_1 = \{(A_2,\allowbreak B_3),\allowbreak (A_4,\allowbreak B_6),(B_2,B_1),(B_2,B_4),\allowbreak (B_3,A_2),(B_3,B_2),(B_5,B_4),\allowbreak (B_6, \allowbreak A_4), (B_6,\allowbreak B_5)\}$. 
\end{example}

%
%
%

Dung-style semantics~\cite{Dung95} can be applied to an AF, to determine the sets of arguments (called \emph{extensions}) that can be accepted. 

\begin{definition}
  \label{def:extensions}
  Given an argumentation framework $\calAF = \tuple{\Args,\attack}$, 
  \begin{itemize} 
    \item $\sfS\subseteq\Args$ \emph{attacks} $A\in\Args$ if there is an $A'\in\sfS$ such that $(A',A)\in\attack$; 
    \item $\sfS$ \emph{defends} $A$ if $\sfS$ attacks every attacker of $A$; 
    \item $\sfS$ is \emph{conflict-free} if there are no $A_1,A_2\in\sfS$ such that $(A_1,A_2)\in\attack$; and 
    \item $\sfS$ is \emph{admissible} ($\Adm$) if it is conflict-free and it defends all of its elements. 
  \end{itemize}
  We denote by $\sfS^+$ the set of all arguments attacked by $\sfS$. An admissible set that contains all the arguments that it defends is a \emph{complete extension} ($\Cmp$). 
  \begin{itemize}
      \item The \emph{grounded extension} ($\Grd$) of $\calAF$ is the minimal (with respect to $\subseteq$) complete extension; 
      \item A \emph{preferred extension} ($\Prf$) of $\calAF$ is a maximal (with respect to $\subseteq$) complete extension; and 
      \item A \emph{semi-stable extension} ($\Sstb$) of $\calAF$ is a complete extension $\sfS$ where $\sfS\cup\sfS^+$ is maximal. 
  \end{itemize}
  An extension will be denoted by $\ext$ and $\Sem(\calAF)$ denotes the set of all the extensions of $\calAF$ under the semantics $\Sem\in\{\Adm,\allowbreak\Cmp,\allowbreak\Grd,\allowbreak\Prf,\allowbreak\Sstb\}$. 
\end{definition}

\noindent
In what follows, given an argumentation framework $\calAF$, we will denote:
\begin{itemize} 
    \item $\Sem\ExtWith(A) = \{\ext\in\Sem(\calAF)\mid A\in\ext\}$ the set of all $\Sem$-extensions of $\calAF$ of which $A$ is a member and 
    \item $\Sem\ExtWithout(A) = \{\ext\in\Sem(\calAF)\mid A\notin\ext\}$ the set of all $\Sem$-extensions of $\calAF$ of which $A$ is not a member.
\end{itemize}

\begin{definition}
  \label{def:Acceptance}
  Let $\calAF = \tuple{\Args,\attack}$ be an argumentation framework, $A\in\Args$ and $\Sem\in\{\Adm,\allowbreak \Grd,\allowbreak \Cmp,\allowbreak \Prf,\allowbreak \Sstb\}$. It is said that, for $\Sem(\calAF)\neq\emptyset$, $A$ is, w.r.t.\ $\Sem$:
  \begin{itemize}
    \item \emph{skeptically accepted} iff $\Sem\ExtWith(A) = \Sem(\calAF)$; 
    \item \emph{credulously accepted} iff $\Sem\ExtWith(A) \neq \emptyset$; 
    \item \emph{not skeptically accepted} iff $\Sem\ExtWithout(A) \neq\emptyset$; 
    \item \emph{not credulously accepted} iff $\Sem\ExtWithout(A) = \Sem(\calAF)$. 
  \end{itemize}
  We will denote skeptical [resp.\ credulous] (non-)acceptance by $\cap$ [resp.\ $\cup$] and when $\cap$ or $\cup$ is clear from the context or not relevant simply write \emph{accepted} and \emph{non-accepted}.
\end{definition}

The notions of attack and defense can also be defined between arguments and can be generalized to indirect versions: given an argumentation framework $\calAF = \tuple{\Args,\attack}$: 
\begin{itemize} 
    \item $A\in\Args$ \emph{defends} $B\in\Args$ if: there is some $C\in\Args$ such that $(C,B)\in\attack$ and $(A,C)\in\attack$, in this case $A$ \emph{directly defends $B$}; or $A$ defends $C\in\Args$ and $C$ defends $B$, in this case  $A$ \emph{indirectly defends $B$}. It is said that \emph{$A$ defends $B$ in $\ext$} if $A$ defends $B$ and $A\in\ext$. 
    \item Similarly, $A\in\Args$ \emph{attacks} $B\in\Args$ if: $(A,B)\in\attack$, in this case $A$ \emph{directly attacks} $B$; or $A$ attacks some $C\in\Args$ and $C$ defends $B$, in this case $A$ \emph{indirectly attacks} $B$.
\end{itemize}

We will require that an explanation for an argument $A$ is \emph{relevant}, in order to prevent that explanations contain arguments that do not influence the acceptance of $A$.

\begin{definition} 
    \label{def:relevance}
    Let $\calAF = \tuple{\Args,\attack}$ and $A,B\in\Args$. It is said that $A$ is \emph{relevant} for $B$ if $A$ (in)directly attacks or defends $B$ and $A$ does not attack itself. A set $\sfS\subseteq\Args$ is relevant for $B$ if all of its arguments are relevant for $B$. A relevant argument $A$ for $B$ is \emph{conflict-relevant} for $B$ if $A$ (in)directly attacks $B$ and it is \emph{defending-relevant} for $B$ if $A$ (in)directly defends $B$. 
\end{definition}

\begin{example}
  \label{ex:abstractAFextensions}
  In $\calAF_1$ $A_2$ and $B_3$ attack each other and both defend themselves. Example conflict-free sets are $\{A_2,B_2\}$ and $\{A_2,B_5\}$. There are four preferred and semi-stable extensions: $\ext_1 = \{A_1,\allowbreak A_2,\allowbreak A_3,\allowbreak A_4,\allowbreak A_5,\allowbreak A_6,\allowbreak B_2,\allowbreak B_5\}$, $\ext_2 = \{A_1,\allowbreak A_2,\allowbreak A_3,\allowbreak A_5,\allowbreak A_6,\allowbreak B_2,\allowbreak B_6\}$, $\ext_3 = \{A_1,\allowbreak A_3,\allowbreak A_5,\allowbreak A_4,\allowbreak A_6,\allowbreak B_1,\allowbreak B_3,\allowbreak B_5\}$ and $\ext_4 = \{A_1,\allowbreak A_3,\allowbreak A_5,\allowbreak A_6,\allowbreak B_1,\allowbreak B_3,\allowbreak B_4,\allowbreak B_6\}$ and $\{A_1,A_3,A_5,A_6\}$ is the grounded extension. 

  The arguments $A_1$, $A_3$, $A_5$ and $A_6$ are skeptically accepted and all other arguments are credulously accepted and not skeptically accepted for $\Sem\in\{\Cmp,\allowbreak\Prf,\allowbreak\Sstb\}$. Argument $A_2$ defends itself and $B_2$ directly, it attacks $B_3$ directly and $B_4$ indirectly, it is conflict-relevant for $B_3$ and $B_4$ and defending-relevant for $B_2$. 
\end{example}

\section{The Basic Framework}
\label{sec:Explanations}

In this section we recall the basic framework of explanations from~\cite{BorgB20basic} and present some new results. The explanations in that paper are defined in terms of two functions: $\depth$, which determines the arguments that are in the explanation and $\argdepth$, which determines what elements of these arguments the explanation presents. To avoid clutter, we instantiate $\depth$ immediately with the following functions, while instantiations of $\argdepth$ are discussed in Section~\ref{sec:Structured}:\footnote{We refer the interested reader to~\cite{BorgB20basic,BorgB21NecSuff} for suggestions of other variations of these functions.}

\begin{definition}
  \label{def:GenDefAtt}
  Let $\calAF = \tuple{\Args,\attack}$ be an AF, $A\in\Args$ and $\ext\in\Sem(\calAF)$ for some semantics $\Sem$. Then: 
  \begin{itemize}
    \item $\DefBy(A) = \{B\in\Args \mid B \text{ defends } A\}$ denotes the set of arguments in $\Args$ that (in)directly defend $A$ and $\DefBy(A,\ext) = \DefBy(A)\cap\ext$ denotes the set of arguments that (in)directly defend~$A$ in $\ext$;
    \item $\NotDef(A,\ext) = \{B\in\Args \mid B \allowbreak \text{ attacks }\allowbreak A \allowbreak\text{ and } \allowbreak \ext \allowbreak \text{ provides }\allowbreak\text{no }\allowbreak\text{defense } A \text{ against this }\allowbreak\text{attack}\}$, denotes the set of all (in)di\-rect attackers of $A$ that are not defended by $\ext$.
  \end{itemize}
\end{definition}

All three functions (i.e., $\DefBy(A)$, $\DefBy(A,\ext)$ and $\NotDef(A,\ext)$) result in relevant sets of arguments for $A$. In particular, all arguments in $\DefBy(A)$ and $\DefBy(A,\ext)$ are defending-relevant and all arguments in $\NotDef(A,\ext)$ are conflict-relevant. 

\begin{example}
  \label{ex:Expl:Notation}
  For $\calAF_1$ we have that: $\DefBy(B_4) = \{B_3,B_6\}$ (i.e., the argument for \emph{malafide webshop} is defended by the arguments for \emph{complaint cannot be retracted} and \emph{the url is not known}) and $\DefBy(B_2) = \{A_2\}$ (i.e., the argument that denies the rule $d_1$ is defended by the argument for \emph{retracted complaint}), $\NotDef(B_4,\ext_1) = \{A_2,A_4,\allowbreak B_2,\allowbreak B_5\}$ (i.e., the argument for \emph{malafide webshop} is attacked by the arguments for \emph{retracted complaint} and \emph{registered url} as well as the arguments that deny the rules $d_1$ and $d_3$ and $\ext_1$ does not provide a defense against these attacks) and $\NotDef(B_4,\ext_2) = \{A_2,B_2\}$ (i.e., the argument for \emph{malafide webshop} is attacked by the argument for \emph{retracted complaint} and the argument that denies rule $d_1$ and $\ext_2$ does not provide a defense against these attacks). 
\end{example}

\subsection{Acceptance Explanations}
\label{sec:Expl:Acc}

Let $\calAF = \tuple{\Args,\attack}$ be an AF and let $A\in\Args$. If $A$ is accepted w.r.t.\ a semantics $\Sem\in\{\Grd,\allowbreak \Cmp,\allowbreak \Prf,\allowbreak \Sstb\}$ and an acceptance strategy $\star\in\{\cap,\cup\}$ then an acceptance explanation can be requested. The explanation depends on the acceptance strategy: for a skeptical reasoner the explanation has to account for the acceptance of the argument in each $\Sem$-extension, while for a credulous reasoner explaining the acceptance of the argument in one $\Sem$-extension is sufficient. 

\begin{definition}[Argument acceptance explanation]
  \label{def:Expl:Acc:Argument}
  Let $\calAF = \tuple{\Args,\attack}$ be an AF, let $A\in\Args$ be accepted given some semantics $\Sem$ and an acceptance strategy ($\cap$ or $\cup$). Then:
  \begin{itemize}
    \item $\Sem\Acc^\cap(A) = \bigcup_{\ext\in\Sem(\calAF)}\DefBy(A,\ext)$
    \item $\Sem\Acc^\cup(A) \in \{\DefBy(A,\ext)\mid \ext\in\Sem\ExtWith(A)\}$.
  \end{itemize}
\end{definition}

The $\cap$-explanation returns all the arguments that defend $A$ in at least one of the $\Sem$-extensions, while the $\cup$-explanation is a set of arguments that defend $A$ in one $\Sem$-extension. 

\begin{example}
  \label{ex:Expl:Acc:Argument} 
  In $\calAF_1$ we have that: $\Prf\Acc^\cup(B_4) = \{B_3,B_6\}$ (i.e., \emph{the webshop is malafide} can be credulously accepted under preferred semantics because of the arguments for \emph{the complaint cannot be retracted} and \emph{the webshop is not registered}); and $\Prf\Acc^\cup(B_2) = \{A_2\}$ (i.e., rule $d_1$ can be denied under credulous acceptance and preferred semantics because of the argument for \emph{the complaint is retracted}). There is no non-empty skeptical acceptance explanation in $\calAF_1$. The reason for this is that all the skeptically accepted arguments (recall Example~\ref{ex:abstractAFextensions}) are not attack and therefore result in empty explanations (see Proposition~\ref{prop:Expl:EmptyAcc} below). 
\end{example}

Next we show some properties of the acceptance explanations.  Proposition~\ref{prop:Expl:Acc:Basic} shows that the defending arguments of an argument $A$ also defend the arguments defended by $A$, while Proposition~\ref{prop:Expl:EmptyAcc} shows that an explanation for an argument is only empty when it is not attacked. 

\begin{proposition}
  \label{prop:Expl:Acc:Basic}
  Let $\calAF = \tuple{\Args,\attack}$ be an AF, $\ext\in\Sem(\calAF)$ for $\Sem\in\{\Adm,\allowbreak \Cmp,\allowbreak \Grd,\allowbreak \Prf,\allowbreak \Sstb\}$ and let $A,B\in\Args$. 
  \begin{itemize}
    \item If $A\in\DefBy(B,\ext)$, then $\DefBy(A,\ext)\subseteq\DefBy(B,\ext)$;
    \item If $A\in\DefBy(B,\ext)$ and $B\in\DefBy(A,\ext)$, then $\DefBy(A,\ext) = \DefBy(B,\ext)$.
  \end{itemize}
\end{proposition}

\begin{proof}
 Let $\calAF = \tuple{\Args,\attack}$ be an AF, $\ext\in\Sem(\calAF)$ for $\Sem\in\{\Adm,\Cmp,\Grd,\Prf,\Sstb\}$ and let $A,B\in\Args$. Suppose that $A\in\DefBy(B,\ext)$. By definition of $\DefBy$ it follows that $A\in\ext$. Let $C\in\Args$ such that $C\in\DefBy(A,\ext)$. Then there is some $D\in\Args$ such that $(D,A)\in\attack$ and $C$ defends $A$ against this attack. However, since $A$ defends $B$, it follows that $D$ attacks $B$ as well, from which it follows that $C$ defends $B$ as well. Therefore $C\in\DefBy(B,\ext)$. The second item follows immediately.
\end{proof}

\begin{proposition}
  \label{prop:Expl:EmptyAcc}
  Let $\calAF = \tuple{\Args,\attack}$ be an AF and let $A\in\Args$ be such that $A$ is accepted w.r.t.\ $\Sem\in\{\Adm,\allowbreak \Cmp,\allowbreak \Grd,\allowbreak \Prf,\allowbreak \Sstb\}$ and $\star\in\{\cap,\cup\}$. Then $\Sem\Acc^\star(A) = \emptyset$ iff there is no $B\in\Args$ such that $(B,A)\in\attack$.
\end{proposition}

\begin{proof}
 Let $\calAF = \tuple{\Args,\attack}$ be an AF and let $A\in\Args$ be such that $A$ is accepted w.r.t.\ $\Sem\in\{\Adm,\allowbreak \Cmp,\allowbreak \Grd,\allowbreak \Prf,\allowbreak \Sstb\}$ and $\star\in\{\cap,\cup\}$.
  
 $\Ra\quad$ Suppose that $\Sem\Acc^\star(A) = \emptyset$. Note that for each $\ext\in\Sem\ExtWith(A)$, $\DefBy(A,\ext) = \emptyset$. Hence there is no attacker of $A$ that is defended by an argument from $\ext$. Since $A\in\ext$, $A$ is defended against its attackers. Therefore, $A$ is not attacked at all.
      
 $\Leftarrow\quad$ Now suppose that $A$ is not attacked. Then there is no argument that defends $A$. Therefore, for any $\ext\in\Sem\ExtWith(A)$, $\DefBy(A,\ext) = \emptyset$. It follows that $\Sem\Acc^\star(A) = \emptyset$.
\end{proof}

\subsection{Non-acceptance Explanations}
\label{sec:Expl:NonAcc}

In order to explain a contrast between an accepted and non-accepted argument, we need non-acceptance explanations as well. Therefore, in this section, basic definitions for explanations of non-accepted arguments are recalled. There are again two types of explanations. 

\begin{definition}[Argument non-acceptance explanation]
  \label{def:Expl:NonAcc:Argument}
  Let $\calAF = \tuple{\Args,\attack}$ be an AF, let $A\in\Args$ be an argument that is not accepted w.r.t.\ $\Sem$ and $\star\in\{\cap,\cup\}$. Then:
  \begin{align*}
  \bullet\ \Sem\NotAcc^\cap(A) &= \bigcup_{\ext\in\Sem\ExtWithout(A)}\NotDef(A,\ext)\\
  \bullet\ \Sem\NotAcc^\cup(A) &= \bigcup_{\ext\in\Sem(\calAF)} \NotDef(A,\ext).
  \end{align*}
\end{definition}

Thus, a non-acceptance explanation contains all the arguments that attack $A$ and for which no defense exists in: some $\Sem$-ex\-ten\-sions (for $\cap$) of which $A$ is not a member; all $\Sem$-extensions (for $\cup$). That for $\cap$ only some extensions have to be considered follows since $A$ is not skeptically accepted as soon as $\Sem\ExtWithout(A)\neq \emptyset$, while $A$ is not credulously accepted when $\Sem\ExtWithout(A) =\Sem(\calAF)$. 

\begin{example}
  \label{ex:Expl:NonAcc:Argument}
  For $\calAF_1$, we have that: $\Prf\NotAcc^\cap(B_4) = \{A_2,\allowbreak B_2,\allowbreak A_4,B_5\}$ (i.e., \emph{the webshop is malafide} is not skeptically accepted under preferred semantics because of the argument for \emph{the complaint is retracted} and \emph{the webshop is registered} and the arguments that deny the rules $d_1$ and $d_3$) and $\Prf\NotAcc^\cap(B_2) = \{B_3\}$ (i.e., the argument that denies rule $d_1$ is not skeptically accepted under preferred semantics because of the argument for \emph{the complaint cannot be retracted}).
\end{example}

The next proposition, the counterpart of Proposition~\ref{prop:Expl:EmptyAcc}, shows that a non-acceptance explanation is never empty.

\begin{proposition}
  \label{prop:Expl:EmptyNonAcc}
  Let $\calAF = \tuple{\Args,\attack}$ be an AF and let $A\in\Args$ be such that $A$ is non-accepted w.r.t.\ $\Sem\in\{\Cmp,\Grd,\Prf,\Sstb\}$ and $\star\in\{\cap,\cup\}$. Then $\Sem\NotAcc^\star(A) \neq \emptyset$.
\end{proposition}

\begin{proof}
  Let $\calAF = \tuple{\Args,\attack}$ and $A\in\Args$ be such that $A$ is non-accepted w.r.t.\ $\Sem\in\{\Cmp,\allowbreak \Grd,\allowbreak \Prf,\allowbreak \Sstb\}$ and $\star\in\{\cap,\cup\}$. Assume that $\Sem\NotAcc^\star(A)  = \emptyset$, then there is no argument $B\in\bigcup_{\ext\in\Sem\ExtWithout(A)}\NotDef(A,\ext)$. It follows that for each extension $\ext\in\Sem\ExtWithout(A)$, $\NotDef(A,\ext) = \emptyset$. Hence there is no $B\in\Args$ such that $(B,A)\in\attack$. But then, by the completeness of $\ext$ it follows that $A\in\ext$. A contradiction. Therefore $\Sem\NotAcc^\star(A)\neq\emptyset$.
\end{proof}

That the above proposition does not hold for $\Sem = \Adm$ follows since not every admissible extension contains all the arguments that it defends. Take for example an AF with arguments $A$ and $B$ and no attacks between them. Then $\{B\}$ is an admissible extension, thus $A$ is not skeptically accepted, yet $\NotDef(A,\{B\}) = \emptyset$. In fact: $\Adm\NotAcc^\cap(A) = \Adm\NotAcc^\cap(B) = \emptyset$. 

For the non-acceptance counterpart of Proposition~\ref{prop:Expl:Acc:Basic} note that $A\in\NotDef(B,\ext)$ entails that $A$ (in)directly attacks $B$. Therefore, if $A$ is not accepted either, the arguments in $\NotDef(A,\ext)$ (in)directly defend $B$. In Section~\ref{sec:ComparingAccNonAcc} we study how acceptance and non-acceptance are related. 

\subsection{Comparing Acceptance and Non-acceptance}
\label{sec:ComparingAccNonAcc}

When looking at Examples~\ref{ex:Expl:Acc:Argument} and~\ref{ex:Expl:NonAcc:Argument} for $B_2$ and $B_4$ one can observe that acceptance and non-acceptance explanations are related. In this section we formalize this observation. In particular, we show that non-acceptance explanations contain the acceptance explanations of (1) the direct attackers; (2) the directly attacked arguments; and (3) the indirectly attacked arguments.

\begin{proposition}
  \label{prop:AccisNonAcc}
  Let $\calAF = \tuple{\Args,\attack}$ be an argumentation framework, let $\ext\in\Sem(\calAF)$ for some $\Sem\in\{\Adm,\allowbreak \Cmp,\allowbreak \Grd,\allowbreak \Prf,\allowbreak \Sstb\}$ and let $A,\allowbreak B_1,\allowbreak \ldots,\allowbreak B_n, \allowbreak C_1,\allowbreak \ldots,\allowbreak C_k\in\Args$ such that $(B_1,A),\allowbreak\ldots,\allowbreak (B_n,A)\in\attack$ and $A$ indirectly attacks $C_1,\ldots,C_k$. Then:
  \begin{enumerate}
    \item  for $B_1,\ldots,B_m\in\ext$, $m\leq n$ it holds that: $\NotDef(A,\ext) \supseteq \DefBy(B_1,\ext)\cup\ldots\cup\DefBy(B_m,\ext)$; \label{item:NotDefAttacked}
    \item when $A\in\ext$: $\DefBy(A,\ext) \subseteq \NotDef(B_1,\ext)\allowbreak \cup\allowbreak \ldots\allowbreak \cup\allowbreak \NotDef(B_n,\allowbreak\ext)$; \label{item:DefinExt}
    \item where $A\in\ext$ and $C_1,\ldots,C_j\notin\ext$, $j\leq k$: $\DefBy(A,\ext)\subseteq\NotDef(C_i,\ext)$ for all $i\in\{1,\ldots,j\}$. \label{item:DefindirAtt}
  \end{enumerate}
\end{proposition}

\begin{proof}
 Let $\calAF = \tuple{\Args,\attack}$ be an AF, $\ext\in\Sem(\calAF)$ for some $\Sem\in\{\Adm,\Cmp,\Grd,\Prf,\Sstb\}$ and $A\in\Args$.
 \begin{enumerate}
    \item Let $B\in\ext$ be such that $(B,A)\in\attack$. If $\DefBy(B,\ext) = \emptyset$, we are done, hence, let $C\in\DefBy(B,\ext)$. Then, by the proof of Proposition~\ref{prop:Expl:EmptyAcc} there is some $D\in\Args$ such that $(D,B)\in\attack$ and $C$ (in)directly attacks $D$. Since $B$ attacks $A$, it follows that $D$ defends $A$ and that $C$ (in)directly attacks $A$. Since $C\in\ext$, $\ext$ does not defend $A$ against the attack from $C$ and therefore $C\in\NotDef(A,\ext)$. 
    \item Let $C\in\DefBy(A,\ext)$, then, by Proposition~\ref{prop:Expl:EmptyAcc}, $n\neq 0$. Suppose that $C$ directly defends $A$, then there is a $B_i\in\{B_1,\ldots,B_n\}$ such that $(C,B_i)\in\attack$. Since $C\in\ext$ it follows that $C\in\NotDef(B_i,\allowbreak\ext)$. Now suppose that $C$ indirectly defends $A$. Then there are $D_1,D_2,\ldots,D_{k}\in\Args$, where $k$ is odd, such that $(D_1,B_i),(D_2,D_1),\ldots(D_k,D_{k-1}),\allowbreak (C,\allowbreak D_k)\in\attack$. Since $D_k$ defends $B_i$ and $C$ attacks $D_k$ it follows that $C$ attacks $B_i$ as well. Hence $C\in\NotDef(B_i,\ext)$. Note that, for any $D\in\DefBy(A,\ext)$ a $B_i\in\{B_1,\ldots, B_n\}$ exists. Therefore $\DefBy(A,\ext)\subseteq\NotDef(B_1,\ext)\cup\ldots\cup\NotDef(B_n,\ext)$. 
    \item Let $A\in\ext$ and suppose that $C_1,\ldots,C_j\notin\ext$ for some $j\leq k$. By assumption $A$ indirectly attacks $C_i$ for all $i\in\{1,\ldots,j\}$ and since $A\in\ext$, $C_i$ is not defended against this attack by $A$. Therefore $A\in\NotDef(C_i,\ext)$. Note that any $D\in\DefBy(A,\ext)$ defends $A$ and therefore indirectly attacks $C_i$ as well. It therefore follows that $D\in\NotDef(C_i,\ext)$ and hence $\DefBy(A,\ext)\subseteq\NotDef(C_i,\ext)$ for all $i\in\{1,\ldots,j\}$.  \qedhere
 \end{enumerate}
\end{proof}

To see that $\NotDef(A,\ext)\not\subseteq\DefBy(B_1,\ext)\cup\ldots\cup\DefBy(B_n,\ext)$, take a look at the following example. Intuitively this is the case since, in terms of labeling semantics~\cite{BaroniCG18SemHandbook}, an argument can be \texttt{in} the extension, attacked by the extension (i.e., \texttt{out}) or attacked by an argument that is not \texttt{in} or \texttt{out} (i.e., \texttt{undecided}). 

\begin{example}
  \label{ex:counter:AccisNonAcc}
  Let $\calAF_2 = \tuple{\Args_2,\attack_2}$, as shown in Figure~\ref{fig:counter:AccisNonAcc:1}. There are two preferred extensions: $\Prf(\calAF_3) = \{\{A\},\allowbreak \{B\}\}$. Here we have that $\NotDef(C,\{B\}) = \{B,D,E,F\}$ but only $(B,C)\in\attack$ such that $B\in\ext$, for which: $\DefBy(B,\{B\}) = \{B\}$. 
  
  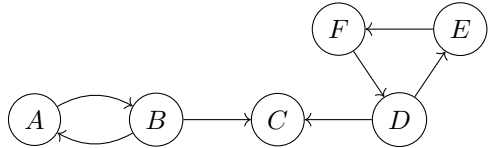
\begin{figure}[ht]
      \centering
      \begin{tikzpicture}[scale=0.8]
        \node[draw,circle] at (-4,0) (A3) {$A$};
        \node[draw,circle] at (-2,0) (B3) {$B$};
        \node[draw,circle] at (0,0) (C3) {$C$};
        \node[draw,circle] at (2,0) (D3) {$D$};
        \node[draw,circle] at (3,1.5) (E3) {$E$};
        \node[draw,circle] at (1,1.5) (F3) {$F$};
        
        
        \draw[->] (A3) to [bend left] (B3);
        \draw[->] (B3) to [bend left] (A3);
        \draw[->] (B3) -- (C3);
        \draw[->] (D3) -- (C3);
        \draw[->] (D3) -- (E3);
        \draw[->] (E3) -- (F3);
        \draw[->] (F3) -- (D3);
      \end{tikzpicture}
      \caption{Graphical representations of the AF $\calAF_2$.}
      \label{fig:counter:AccisNonAcc:1}
  \end{figure}
\end{example}

Knowing how acceptance and non-acceptance explanations are related is useful in the context of contrastive explanations, where explanations are not only about the requested argument, but about arguments that are conflicting with the requested argument as well.

\section{Contrastive Explanations}
\label{sec:Contrastive}

A contrastive explanation explains $A$ by explaining \emph{why $A$ rather than $B$}. Important in contrastive explanations is that the difference between fact (i.e., $A$) and foil (i.e., $B$) is highlighted. In this paper we assume that fact and foil are not always compatible: it cannot be the case that both $A$ and $B$ are skeptically accepted. Intuitively, we make this assumption since otherwise there is no contrastive question for fact and foil (i.e., \emph{why both $A$ and $B$} is not contrastive). 

In the context of formal argumentation contrastive explanations are modeled by comparing the elements of the basic explanations that explain the acceptance [resp.\ non-acceptance] of the fact and, at the same time, explain the non-acceptance [resp.\ acceptance] of the foil. Hence, the main idea of the introduced contrastive explanations will be that these return the \emph{common elements} of the basic acceptance [resp.\ non-acceptance] explanation of the fact and the basic non-acceptance [resp.\ acceptance] explanation of the foil. Recall the example scenario from Section~\ref{sec:Applying}. For a webshop to be malafide ($B_4$ is accepted) there should be no exceptions to rules $d_1$ (an investigation is done when a complaint is filed) and $d_3$ (the webshop is malafide if the url is suspicious), i.e., arguments $B_2$ and $B_5$ should not be accepted. The absence of an exception is therefore a good example of a foil: \emph{why is the webshop malafide, rather than that there is an exception to $d_1$?} can be answered with: \emph{since the owner is known by the police} (i.e., $A_5$ and $B_3$), which prevents the possible exception to $d_1$ and, similarly, \emph{why is the webshop malafide, rather than that there is an exception to $d_3$?} can be answered with: \emph{since the registration of the url was recently retracted} (i.e., $A_6$ and $B_6$), which prevents the possible exception to $d_3$. Note that contrastiveness is a selection mechanism~\cite{Miller19}: by choosing a foil, the explanation of the fact is reduced from all possible explanations (i.e., the basic explanations) to those parts of the explanation that answer the contrastive question.  

\begin{definition}[Contrastive explanations]
  \label{def:Expl:Contrastive:ExplicitFoil}
  Let $\calAF =\tuple{\Args,\attack}$ be an AF, let $A\in\Args$ (the fact) and let $\sfS\subseteq \Args$ (a set of foils) such that there is no $\ext\in\Sem(\calAF)$ in which $A,B\in\ext$ for all $B\in\sfS$. Moreover, $\star\in\{\cap,\cup\}$ and $\dag = \cap$ if $\star = \cup$ and $\dag = \cup$ if $\star = \cap$. 
  Contrastive explanations are then defined as in  Figure~\ref{fig:ContrastiveDef}. 
\end{definition}

  \begin{figure*}[t]
  \footnotesize
  \begin{numcases}{\Sem\Cont^\star(A,\sfS)=} \Sem\Acc^\star(A)\cap\bigcup_{B\in\sfS}\Sem\NotAcc^\dag(B) & $\text{ if } \Sem\Acc^\star(A)\cap\bigcup_{B\in\sfS}\Sem\NotAcc^\dag(B) \neq\emptyset$  \label{Cont:Acc:relevant}\\ \tuple{\Sem\Acc^\star(A),\bigcup_{B\in\sfS}\Sem\NotAcc^\dag(B)} & \text{ otherwise}. \label{Cont:Acc:nonrelevant} \end{numcases}
  
  \begin{numcases}{\Sem\ContN^\star(A,\sfS) =}
  \Sem\NotAcc^\star(A)\cap\bigcup_{B\in\sfS}\Sem\Acc^\dag(B)  & $\text{ if } \Sem\NotAcc^\star(A)\cap\bigcup_{B\in\sfS}\Sem\Acc^\dag(B) \neq\emptyset$  \label{Cont:NonAcc:relevant}\\ \tuple{\Sem\NotAcc^\star(A),\bigcup_{B\in\sfS}\Sem\Acc^\dag(B)} & \text{ otherwise}. \label{Cont:NonAcc:nonrelevant}
  \end{numcases}
  \caption{Definition of contrastive explanations (see Definition~\ref{def:Expl:Contrastive:ExplicitFoil}).}
  \label{fig:ContrastiveDef}
  \end{figure*}

In words, when there are arguments that cause the fact to be accepted [resp.\ non-accepted] and the foil to be non-accepted [resp.\ accepted], the contrastive explanation is the set of such arguments Line~\eqref{Cont:Acc:relevant} [resp.\ Line~\eqref{Cont:NonAcc:relevant}]. If there are no common causes for the acceptance [resp.\ non-acceptance] of the fact and the non-acceptance [resp.\ acceptance] of the foil, the contrastive explanation is a pair of the respective explanations Line~\eqref{Cont:Acc:nonrelevant} [resp.\ Line~\eqref{Cont:NonAcc:nonrelevant}]. 

\begin{example}
  \label{ex:Expl:Contrastive:Combining:Formula}
  For our running example with $\calAF_1$ we have that: 
  \begin{itemize}
      \item $\Prf\Cont^\cup(B_4,B_2) = \{B_3\}$: \emph{the webshop is malafide} rather than that $d_1$ is denied because there is an argument for \emph{the complaint cannot be retracted}, 
      \item $\Prf\Cont^\cup(B_4,B_5) = \{B_6\}$: \emph{the webshop is malafide} rather than that $d_3$ is denied because there is an argument for \emph{the webshop is not registered} and 
      \item $\Prf\Cont^\cup(B_4,\{B_2,B_5\}) = \{B_3,B_6\}$: \emph{the webshop is malafide} rather than that the rules $d_1$ and $d_3$ are denied because there are arguments for \emph{the complaint cannot be retracted} and \emph{the webshop is not registered}, while 
      \item $\Prf\ContN^\cap(B_4,B_2) = \{A_2\}$: \emph{the webshop is not malafide} and the rule $d_1$ is denied because \emph{the complaint was retracted};
      \item $\Prf\ContN^\cap(B_4,B_5) = \{A_4\}$: \emph{the webshop is not malafide} and the rule $d_3$ is denied because \emph{the webshop is registered}.
  \end{itemize} 

  Recall (Example~\ref{ex:Expl:Acc:Argument}) that the acceptance of $B_4$ can be explained by $B_3$ and $B_6$, when compared to the non-acceptance of $B_2$ [resp.\ $B_5$] the acceptance of $B_4$ is explained by $B_3$ [resp.\ $B_6$] alone. 
\end{example}

Note that contrastive explanations are not necessarily unique. This follows since the credulous acceptance explanation (recall Definition~\ref{def:Expl:Acc:Argument}) is not necessarily unique and the contrastive explanations might be constructed from such explanations.

One could consider these explanations more meaningful when they return a set, rather than a pair. This is the case since then there are arguments that influence both the acceptance [resp.\ non-acceptance] of the fact and the non-acceptance [resp.\ acceptance] of the foil. If the explanation would be a pair, it would essentially be a combination of the separate basic explanations for fact and foil and provides no meaningful extra information on top of the two basic (non-contrastive) explanations. The next proposition shows that in most cases the explanation is a set. Only when the accepted argument is not attacked or fact and foil are not conflict-relevant is the intersection empty. 

\begin{proposition}
  \label{prop:Expl:EmptyInter}
  Let $\calAF = \tuple{\Args,\attack}$ be an an argumentation framework and $A,B\in\Args$. If $\Sem\Acc^\star(A)\cap\Sem\NotAcc^\dag(B) =\emptyset$ then $\Sem\Acc^\star(A) = \emptyset$; or $A$ is not conflict-relevant for $B$.
\end{proposition}

\begin{proof}
  Let $\calAF = \tuple{\Args,\attack}$ be an argumentation framework and let $A,B\in\Args$. When $\Sem\Acc^\star(A) = \emptyset$ it follows immediately that $\Sem\Acc^\star(A)\cap\Sem\NotAcc^\dag(B) = \emptyset$. By Proposition~\ref{prop:Expl:EmptyNonAcc} $\Sem\NotAcc^\dag(B)\neq\emptyset$. Suppose that $A$ is conflict-relevant for $B$ and, without loss of generality, that $\Sem\Acc^\star(A)\neq\emptyset$. Since $\Sem\Acc^\star(A)$ is requested and, by assumption $\Sem\Acc^\star(A)\neq\emptyset$, there is some $\ext\in\Sem(\calAF)$ such that $A\in\ext$ and $\DefBy(A,\ext)\neq\emptyset$. If $(A,B)\in\attack$, by Proposition~\ref{prop:AccisNonAcc}.\ref{item:NotDefAttacked} $\DefBy(A,\ext)\subseteq\NotDef(B,\ext)$. If $A$ indirectly attacks $B$, then by Proposition~\ref{prop:AccisNonAcc}.\ref{item:DefindirAtt} we have that $\DefBy(A,\ext)\subseteq\NotDef(B,\ext)$ as well. Since it holds that $\DefBy(A,\ext)\neq\emptyset$ we have $\Sem\Acc^\star(A)\cap\NotAcc^\dag(B)\neq\emptyset$. 
\end{proof}

In view of the above result, the following conditions are introduced on the fact and foil. By requiring these conditions to hold, meaningful contrastive explanations can be obtained. For this let $\calAF = \tuple{\Args,\attack}$ be an AF and let $\{A\}\cup \sfS\subseteq\Args$. Then $\Sem\Cont^\star(A,\sfS)$ [resp.\ $\Sem\ContN^\star(A,\sfS)$] can be requested when, for each $B\in\sfS$:
\begin{itemize}
    \item $A$ is at least credulously accepted [resp.\ not skeptically accepted] and $B$ is at least not skeptically accepted [resp.\ credulously accepted];
    \item for each $\ext\in\Sem(\calAF)$ it never holds that $\{A,B\}\subseteq\ext$;
    \item either $A$ is conflict-relevant for $B$ or $B$ is conflict-relevant for $A$. 
\end{itemize}
These conditions ensure that fact and foil are incompatible, but still relevant for each other: it is explained what makes the fact accepted [resp.\ non-accepted] and, simultaneously causes the foil to be non-accepted [resp.\ accepted]. This prevents contrastive explanations for arguments that are not related or conflicting. These conditions are not exhaustive, depending on, e.g., the application, a user might wish to enforce further conditions on fact or foil. 

\subsection{Non-explicit Foil}
\label{sec:Contrastive:NonExplicitFoil}

When humans request a (contrastive) explanation the foil is sometimes left implicit, yet the expected explanation does not provide all reasons for the fact happening, but should rather explain the difference between fact and foil. While humans are able to detect the foil based on, e.g., context, this is a challenge for AI systems, including argumentation. In particular, it is impossible to provide one strategy, since different applications entail different foils. For example, given a fact $A$, if argumentation is applied to determine a yes or no answer (e.g., whether one qualifies for a loan), the foil would be \emph{not A}, but if the foil should be chosen from a larger set (e.g., a medical diagnosis), it might be any member of that set.

Since in the definition of contrastive explanations it is necessary to provide a foil, a way to determine the foil is required. This is where one of the advantages of formal argumentation comes in: the explicit nature of conflicts between arguments makes that the foil or a set of foils can be constructed from an AF. Since the relation between arguments is only determined by the attack relation in our setting, it is impossible to distinguish between attackers. To illustrate the possibilities, in the remainder of the paper the foil will consist of all directly attacking arguments. 

\begin{definition}
  \label{def:Expl:Foil:Abstract}
  Let $\calAF = \langle\Args,\allowbreak\attack\rangle$ be an AF and let $A\in\Args$. Then: $\Foil(A) = \{B\in\Args\mid B\text{ directly attacks }A\}$. 
\end{definition}

\begin{example}
  \label{ex:Expl:Foil:Abstract}
  For the framework $\calAF_1$ we have that: $\Foil(B_4) = \{B_2,B_5\}$, \emph{the webshop is malafide} is in direct conflict with the arguments that deny the rules $d_1$ and $d_3$; $\Foil(B_2) = \{B_3\}$, the argument that denies rule $d_1$ is in direct conflict with \emph{the complaint cannot be retracted} and $\Foil(B_5) = \{B_6\}$, the argument that denies rule $d_3$ is in direct conflict with \emph{the webshop is not registered}. 
\end{example}

Note that, for our running example, the explanations with implicit foil do not change:

\begin{example}
  \label{ex:Expl:Foil:Congruent}
  For the AF $\calAF_1$:  $\Prf\Cont^\cup(B_4,\Foil(B_4)) = \{B_3,B_6\}$ and $\Prf\ContN^\cap(B_2,\Foil(B_2)) = \{B_3\}$. These correspond to the explanations from Example~\ref{ex:Expl:Contrastive:Combining:Formula}.
\end{example}

In what follows it will be assumed that $\Foil(A)\neq\emptyset$, for fact $A$, i.e., that a foil exists. Note that, by Definition~\ref{def:Expl:Foil:Abstract}, for any AF $\calAF = \tuple{\Args,\attack}$ and $A\in\Args$, $\Foil(A) = \emptyset$ iff there is no $B\in\Args$ such that $(B,A)\in\attack$. Hence, any argument without a foil is not attacked at all. In such a case a non-acceptance explanation is not applicable and, by Proposition~\ref{prop:Expl:EmptyAcc}, the acceptance explanation is empty. Therefore, this requirement does not restrict our results.

The next proposition shows that the obtained contrastive explanations are meaningful when the first condition of the applicability of contrastive explanations is fulfilled and the foil is defined as in Definition~\ref{def:Expl:Foil:Abstract}. 

\begin{proposition}
  \label{prop:Expl:Foil:ApplyContrastive}
  Let $\calAF = \tuple{\Args,\attack}$ be an AF, let $A\in\Args$ be such that $\Foil(A)\neq\emptyset$ and $\Sem\in\{\Adm,\Cmp,\Grd,\Prf,\Sstb\}$. Then a contrastive acceptance [resp.\ non-acceptance] explanation can be requested for $A$, when $A$ is at least credulously accepted [resp.\ not skeptically accepted] and for all $B\in\Foil(A)$, $B$ is at least not skeptically accepted [resp.\ credulously accepted]. 
\end{proposition}

\begin{proof}
  Let $\calAF = \tuple{\Args,\attack}$ be an AF, let $A\in\Args$ be such that $\Foil(A)\neq\emptyset$ and $\Sem\in\{\Adm,\allowbreak \Cmp,\allowbreak \Grd,\allowbreak \Prf,\allowbreak \Sstb\}$. To show that:
  \begin{enumerate}
      \item for each $\ext\in\Sem(\calAF)$ it is never the case that $\{A,B\}\subseteq\ext$: by definition, $\Foil(A) = \{B\in\Args\mid (B,A)\in\attack\}$ and hence, since each $\ext\in\Sem(\calAF)$ is conflict-free, it is never the case that $\{A,B\}\subseteq\ext$.
      \item either $A$ is conflict-relevant for $B$ or $B$ is conflict-relevant for $A$: by definition, $(B,A)\in\attack$ and hence $B$ is conflict-relevant for $A$. \qedhere
  \end{enumerate}
\end{proof}

In view of the above proposition we obtain the following corollary from Propositions~\ref{prop:Expl:EmptyInter} and~\ref{prop:Expl:Foil:ApplyContrastive}. 

\begin{corollary}
  \label{cor:Expl:Foil:Alternative}
  Let $\calAF = \tuple{\Args,\attack}$ be an AF, let $A\in\Args$ be such that $\Foil(A)\neq\emptyset$ and $\Sem\in\{\Adm,\allowbreak \Cmp,\allowbreak \Grd,\allowbreak \Prf,\allowbreak \Sstb\}$. Then:
  \begin{itemize}
    \item the contrastive explanation $\Sem\Cont^\star(A,\Foil(A))$ is never a pair, i.e., is not of the form $\tuple{\Sem\Acc^\star(A),\allowbreak \bigcup_{B\in \Foil(A)}\Sem\NotAcc^\dag(B)}$; 
    \item $\Sem\ContN^\star(A,\allowbreak\Foil(A))=\left\langle\Sem\NotAcc^\star(A),\allowbreak \bigcup_{B\in \Foil(A)}\Sem\Acc^\dag(B)\right\rangle$ iff $\Sem\Acc^\dag(B) = \emptyset$ for all $B\in\Foil(A)$.
  \end{itemize}
\end{corollary}

Thus, when the foil is determined as in Definition~\ref{def:Expl:Foil:Abstract}, non-acceptance contrastive explanations are pairs if and only if the fact is only attacked by non-attacked arguments.

\section{Contrastive Explanations in Structured Argumentation}
\label{sec:Structured}

Since many approaches to structured argumentation result in an abstract argumentation framework (see e.g.,~\cite{BesnardGHMPS14intro}), the basic and contrastive explanations as well as the results in this paper are applicable to such approaches as well. However, like in~\cite{BorgB20basic}, the structure of the arguments within any approach to structured argumentation, makes it possible to refine the explanations. For this we take ASPIC$^+$~\cite{Prakken10}.\footnote{For the sake of simplicity and conciseness we take classical negation (denoted by $\neg$) as the contrariness function and we do not consider preferences in this paper.} 

\paragraph{\bf ASPIC$^+$} In ASPIC$^+$, an \emph{argumentation system} $\AS=\tuple{\calL,\calR,n}$ consisting of a propositional language $\calL$, a set of rules $\calR = \calR_s\cup\calR_d$ (of the form $r = \phi_1,\ldots,\phi_n\rightarrow\psi$ for strict rules ($r\in\calR_s$) and $r=\phi_1\ldots,\phi_n\Ra\psi$ for defeasible rules ($r\in\calR_d$)) such that $\calR_s\cap\calR_d = \emptyset$ and the naming convention $n:\calR_d\rightarrow\calL$ for defeasible rules and the \emph{knowledge base} $\calK = \calK_n\cup\calK_p$ (containing the disjoint sets of axioms ($\calK_n$) and ordinary premises ($\calK_p$)) form an \emph{argumentation theory} $\AT = \tuple{\AS,\calK}$, within which arguments can be constructed:

\begin{definition}
  \label{def:ASPICArgu}
  An \emph{argument} $A$ on the basis of a knowledge base $\calK$ in an argumentation system $\ArgSys$ is:
  \begin{enumerate}
    \item $\phi$ if $\phi\in\calK$, with $\prem(A) = \sub(A) = \{\phi\}$, $\conc(A) = \phi$ and $\tprule(A) = \text{undefined}$;
    \item $A_1,\ldots, A_n \rightarrow/\Ra \psi$ if $A_1, \ldots, A_n$ are arguments such that there exists a strict/de\-fea\-si\-ble rule $\conc(A_1),\allowbreak \ldots,\allowbreak \conc(A_n) \rightarrow/\Ra\psi$ in $\calR_s/\calR_d$. 
    
    $\prem(A)= \prem(A_1)\cup\ldots\cup\prem(A_n)$; $\conc(A) = \psi$; $\sub(A) = \sub(A_1)\cup\ldots\cup\sub(A_n)\cup\{A\}$; 
    $\tprule(A) = \conc(A_1), \ldots,\conc(A_n)\rightarrow/\Ra\psi$.
  \end{enumerate}
  For a set of argument $\sfS$: $\prem(\sfS) = \{\prem(A) \mid A\in\sfS\}$ and $\concs(\sfS) = \{\conc(A) \mid A\in\sfS\}$. 
\end{definition}

Attacks on an argument are based on the rules and premises applied in the construction of that argument. In what follows we let $\phi = -\psi$ if $\phi = \neg \psi$ or $\psi = \neg \phi$.

\begin{definition}
  \label{def:ASPIC:att}
  An argument $A$ \emph{attacks} an argument $B$ iff $A$ \emph{undercuts}, \emph{rebuts} or \emph{undermines} $B$, where:
  \begin{itemize}
    \item $A$ \emph{undercuts} $B$ (on $B'$) iff $\conc(A) = -n(r)$ for some $B'\in\sub(B)$ with $r=\tprule(B')\in\calR_d$;
    \item $A$ \emph{rebuts} $B$ (on $B'$) iff $\conc(A) = -\phi$ for some $B' \in\sub(B)$ of the form $B''_1,\ldots, B''_n\Ra\phi$;
    \item $A$ \emph{undermines} $B$ (on $\phi$) iff $\conc(A)= -\phi$ for some $\phi\in\prem(B)\setminus \calK_n$.
  \end{itemize}
\end{definition}

\emph{Abstract argumentation frameworks} can be derived from argumentation theories: $\calAF(\AT) = \tuple{\Args,\attack}$, where $\Args$ is the set of all arguments constructed from the argumentation theory $\AT$ and $(A,B)\in\attack$ iff $A$ attacks $B$ according to Definition~\ref{def:ASPIC:att}.

Dung-style semantics can be applied to such AFs. We denote: $\All\Args(\phi) = \{A\in\Args\mid\conc(A) = \phi\}$, $\Sem\Accept(\phi) = \All\Args(\phi)\cap\bigcup\Sem(\calAF)$, $\Sem\ExtWith(\phi) = \bigcup\{\Sem\ExtWith(A)\mid A\in\All\Args(\phi)\}$, $\Sem\ExtWithout(\phi) = \bigcap\{\Sem\ExtWithout(A)\mid A\in\All\Args(\phi)\}$.

The example in this section is the instantiation of the framework $\calAF_1$ from Example~\ref{ex:abstractAF}, based on the scenario from Section~\ref{sec:Applying}. Recall that we had the following abbreviations: \emph{cf} -- a complaint is filed; \emph{rc} -- the complaint is retracted; \emph{sa} -- the url is suspicious; \emph{ka} -- the webshop is registered; \emph{kp} -- the owner of the webshop is known by the police; \emph{rr} -- the registration of the webshop was recently retracted; \emph{iw} -- an investigation into the webshop is done; \emph{m} -- the webshop is malafide.

\begin{example}
  \label{ex:Structured2}
  Let $\AT_1 = \tuple{\AS,\calK}$, with $\AS = \tuple{\calL,\calR,n}$ where $\calR$ and $\calK = \calK_p = \{\textit{cf},\textit{rc},\textit{sa},\textit{ka},\textit{kp},\allowbreak\textit{rr}\}$ are such that the set of arguments $\Args_1$ that can be constructed from $\AT_1$ is: 
  \begin{align*}
    A_1 :\ & \textit{cf} && A_2 : \textit{rc} && A_3 : \textit{sa} \\ 
    A_4 :\ & \textit{ka} && A_5 : \textit{kp} && A_6 : \textit{rr} \\ 
    B_1 :\ & A_1 \overset{d_1}{\Ra} \textit{iw} && B_2 : A_2 \overset{d_2}{\Ra} \neg n(d_1) && B_3 : A_5 \overset{d_5}{\Ra} \neg\textit{rc}\\  
    B_4 :\ & B_1, A_3 \overset{d_3}{\Ra} m && B_5 : A_4 \overset{d_4}{\Ra}\neg n(d_3) && B_6 : A_6 \overset{d_6}{\Ra} \neg \textit{ka}. 
  \end{align*}
  Figure~\ref{fig:Webshop} shows a graphical representation of the corresponding AF $\calAF(\AT_1)$. Moreover, the extensions are discussed in Example~\ref{ex:abstractAFextensions}. 
\end{example}

%
%

\paragraph{\bf Basic Explanations for ASPIC$^+$} In abstract argumentation the arguments are abstract entities, however, in ASPIC$^+$ the structure of the arguments is known and can be used in the explanations. To this end we use the function $\argdepth$, which determines the content of an explanation (e.g., explanations can consist of arguments ($\argdepth = \id$) or of the premises of those arguments ($\argdepth = \prem$)).\footnote{See~\cite{BorgB20basic} for additional variations for $\argdepth$ and see~\cite{BorgB21intuitive} for a discussion on how these variations can be applied in the context of an example from the Netherlands Police}. 
Formula explanations differ in two ways from the explanations in Section~\ref{sec:Explanations}: the function $\argdepth$ is applied; and the arguments for $\phi$ have to be considered (e.g., all accepted arguments for $\phi$ for $\cap$-acceptance and an accepted argument for $\phi$ for $\cup$-acceptance).
The basic explanations from Section~\ref{sec:Explanations} for formulas are defined by:

\begin{definition}[Basic formula explanations]
  \label{def:BasicExpl:Formula}
  Let $\phi\in\calL$ be a formula and $\calAF(\AT) = \tuple{\Args,\attack}$ be based on $\AT$. Suppose that $\phi\in\calL$ is accepted w.r.t.\ $\Sem$ and $\cap$ or $\cup$. Then:
  \begin{align*}
    \bullet\ &\Sem\Acc^\cap(\phi) = \argdepth\left(\bigcup_{A\in\Sem\Accept(\phi)}\ \bigcup_{\ext\in\Sem(\calAF)}\DefBy(A,\ext)\right); \\
    \bullet\ &\Sem\Acc^\cup(\phi) \in\left\{\argdepth(\DefBy(A,\ext))\mid A\in\Sem\Accept(\phi),\  \ext\in\Sem\ExtWith(A)\right\}.
  \end{align*}
  Suppose now that $\phi$ is non-accepted w.r.t.\ $\Sem$ and $\cap$ or $\cup$:
  \begin{align*}
    \bullet\ &\Sem\NotAcc^\cap(\phi) = \argdepth\left(\bigcup_{A\in\All\Args(\phi)}\ \bigcup_{\ext\in\Sem\ExtWithout(A)}\NotDef(A,\ext)\right);\\
    \bullet\ &\Sem\NotAcc^\cup(\phi) = \argdepth\left(\bigcup_{A\in\All\Args(\phi)}\ \bigcup_{\ext\in\Sem(\calAF)}\NotDef(A,\ext)\right).
\end{align*}
\end{definition}

That for the $\cap$-non-acceptance explanation all arguments for $\phi$ have to be accounted for follows since it might be the case that an explanation does not contain one particular argument for $\phi$ but it does contain another. 

\begin{example}
  \label{ex:BasicExpl:Structured}
  For $\calAF(\AT_1)$ based on $\AT_1$ from Example~\ref{ex:Structured2}:
  \begin{itemize}
    \item $\Prf\Acc^\cup(m) = \{B_3,B_6\}$ for $\argdepth = \id$ and $\Prf\Acc^\cup(m) = \{\textit{kp},\textit{rr}\}$ for $\argdepth = \prem$;
    \item $\Prf\Acc^\cup(\neg n(d_1)) = \{A_2\}$ for $\argdepth = \id$ and $\Prf\Acc^\cup(\neg n(d_1)) = \{\textit{rc}\}$ for $\argdepth = \prem$;
    \item $\Prf\NotAcc^\cap(m) = \{A_2,A_4,B_2,B_5\}$ for $\argdepth = \id$ and for $\argdepth = \prem$ $\Prf\NotAcc^\cap(m) = \{\textit{rc},\textit{ka}\}$. 
  \end{itemize}
  Based on the underlying scenario, we have that: 
\begin{itemize}
  \item The webshop is malafide: $m$ can be  credulously accepted since the owner of the webshop is known by the police ($\textit{kp}$) and the registration at the chamber of commerce was recently retracted ($\textit{rr}$), from which it follows that no exceptions could be derived. 
  \item The webshop is not malafide: $m$ can be not skeptically accepted since the complaint was retracted ($\textit{rc}$) and the url of the webshop is registered ($\textit{ka}$).  
\end{itemize}
\end{example}

Like for arguments (recall Propositions~\ref{prop:Expl:EmptyAcc} and~\ref{prop:Expl:EmptyNonAcc}) we have the following result:
\begin{proposition}
  \label{prop:Expl:Empty:Structured}
  Let $\calAF(\AT) = \tuple{\Args,\attack}$ be an AF based on $\AT$, let $\phi\in\calL$, $\star\in\{\cap,\cup\}$, $\Sem\in\{\Adm,\allowbreak \Cmp,\allowbreak \Grd,\allowbreak \Prf,\allowbreak \Sstb\}$ and $\argdepth$ be such that $\argdepth(\sfS) \neq\emptyset$ when $\sfS\neq\emptyset$:
  \begin{itemize}
      \item when $\phi$ is accepted w.r.t.\ $\star$ and $\Sem$, $\Sem\Acc^\star(\phi) = \emptyset$ iff there is no $B\in\Args$ such that $(B,C)\in\attack$ for any $C\in\Sem\Accept(\phi)$;
      \item when $\phi$ is non-accepted w.r.t.\ $\star$ and $\Sem$, then we have that $\Sem\NotAcc^\star(\phi)\neq\emptyset$.
  \end{itemize}
\end{proposition}
  
\begin{proof}
   Let $\calAF(\AT) = \tuple{\Args,\attack}$ be an AF based on $\AT$, let $\phi\in\calL$, $\star\in\{\cap,\cup\}$, $\Sem\in\{\Adm,\allowbreak \Cmp,\allowbreak \Grd,\allowbreak \Prf,\allowbreak \Sstb\}$ and $\argdepth$ be such that $\argdepth(\sfS) \neq\emptyset$ when $\sfS\neq\emptyset$. Consider both items:
   \begin{itemize}
       \item $\Ra\quad$ Suppose that $\Sem\Acc^\star(\phi) = \emptyset$. Then for each $\ext\in\Sem\ExtWith(\phi)$ and each $A\in\ext\cap\AllArgs(\phi)$: $\DefBy(A,\ext) = \emptyset$. Hence there is no attacker of any argument in $\AllArgs(\phi)$ that is defended by some argument from $\ext$. Since $\phi\in\concs(\ext)$, the arguments in $\AllArgs(\phi)$ are defended against their attackers. It follows that $\AllArgs(\phi)$ is not attacked at all.
      
        $\Leftarrow\quad$ Now suppose that no argument in $\bigcup\Sem\ExtWith(\phi)\cap\AllArgs(\phi)$ is attacked. Then there is no argument that defends an argument for $\phi$. Therefore, for any $\ext\in\Sem\ExtWith(\phi)$ and any $A\in\ext\cap\AllArgs(\phi)$, $\DefBy(A,\ext) = \emptyset$. It follows that $\Sem\Acc^\star(A) = \emptyset$. 
       \item Assume, towards a contradiction, that $\Sem\NotAcc^\star(\phi)  = \emptyset$, then there is no argument $B\in\bigcup_{\ext\in\Sem\ExtWithout(\phi)}\NotDef(A,\ext)$. It follows that for each $\ext\in\Sem\ExtWith(\phi)$ and any $A\in\AllArgs(\phi)$, $\NotDef(A,\ext) = \emptyset$. Hence there is no $B\in\Args$ such that $(B,A)\in\attack$ for any $A\in\AllArgs(\phi)$. But then, by the completeness of $\ext$ it follows that $\AllArgs(\phi)\subseteq\ext$. A contradiction. Therefore $\Sem\NotAcc^\star(\phi)\neq\emptyset$. \qedhere
   \end{itemize}
\end{proof}  

This proposition shows that an acceptance explanation for a formula $\phi$ is only empty when no argument for $\phi$ is attacked, while a non-acceptance explanation for $\phi$ is never empty. 

\paragraph{\bf Contrastive Explanations for ASPIC$^+$} With these basic explanations for formulas,  contrastive explanations can be defined for formulas in a similar way as for arguments (recall Figure~\ref{fig:ContrastiveDef}), where the fact is now a formula and the foil is a set of formulas. When the foil is not explicit, we can define, for example, $\Foil(\phi) = \{-\phi\mid-\phi\in\concs(\Args)\}$. 

In order to obtain meaningful contrastive explanations in the structured setting, we redefine the conditions on the application of contrastive explanations in the context of ASPIC$^+$. Let $\calAF(\AT)$ be an AF based on $\AT$, let $\{\phi\}\cup\sfS\subseteq\concs(\Args)$. Then $\Sem\Cont^\star(\phi,\sfS)$ [resp.\ $\Sem\ContN^\star(\phi,\sfS)$] can be requested when for each $\psi\in\sfS$:
\begin{itemize}
    \item $\phi$ is at least credulously accepted (i.e., $\Sem\ExtWith(\phi)\neq\emptyset$) [resp.\ not skeptically accepted (i.e., $\Sem\ExtWithout(\phi)\neq\emptyset$)] and $\psi$ is at least not skeptically accepted [resp.\ credulously accepted];
    \item for each $\ext\in\Sem(\calAF)$ it is never the case that $\{\phi,\psi\}\subseteq\concs(\ext)$;
    \item either some $A\in\All\Args(\phi)$ is conflict-relevant for some $B\in\All\Args(\psi)$ or some $B\in\All\Args(\psi)$ is conflict-relevant for some $A\in\All\Args(\phi)$. 
\end{itemize}

The basic explanations from Example~\ref{ex:BasicExpl:Structured} are exhaustive: all the reasons why the webshop is (not) malafide are provided. For applications with more arguments, this may result in more reasons within one explanation. With our contrastive explanations, the explanation can focus on an explicit contrastive question: 

\begin{example}
  \label{ex:Contrastive:Structured}
For $\calAF(\AT_4)$, where $\argdepth = \prem$ we have that: 
  \begin{itemize}
    \item $\Prf\Cont^\cup(m, \neg n(d_1)) = \{\textit{kp}\}$: the webshop is malafide rather than that there is an exception to rule $d_1$, since the owner is known by the police ($\textit{kp}$); 
    \item $\Prf\Cont^\cup(m,\neg n(d_3)) = \{\textit{rr}\}$: the webshop is malafide rather than that there is an exception to rule $d_3$, since the registration was recently retracted ($\textit{rr}$); and 
    \item $\Prf\ContN^\cap(m,\neg n(d_1)) = \{\textit{rc}\}$: the webshop is not malafide and an exception to rule $d_1$ applies, since the complaint was retracted ($\textit{rc}$). 
  \end{itemize}
\end{example}

The above example shows that the contrastive explanations are better tailored to one question and result in smaller explanations.

Next we turn to the formula counterparts of the results from Section~\ref{sec:Contrastive}. First 
on empty contrastive explanations. 

\begin{proposition}
  \label{prop:Expl:EmptyInter:Structured}
  Let $\calAF(\AT) = \tuple{\Args,\attack}$ be an AF based on $\AT$ and let $\phi,\psi\in\calL$ such that $\phi,\psi\in\concs(\Args)$, it holds that  $\Sem\Acc^\star(\phi)\cap\Sem\NotAcc^\dag(\psi) = \emptyset$ implies that $\Sem\Acc^\star(\phi) = \emptyset$; or for each $A\in\All\Args(\phi)$ and each $B\in\All\Args(\psi)$, $A$ is not conflict-relevant for $B$. 
\end{proposition}

\begin{proof}
  Let $\calAF(\AT) = \tuple{\Args,\attack}$ be an AF based on $\AT$ and let $\phi,\psi\in\calL$ such that $\phi,\psi\in\concs(\Args)$. That $\Sem\Acc^\star(\phi)\cap\Sem\NotAcc^\dag(\psi) = \emptyset$ when $\Sem\Acc^\star(\phi) = \emptyset$ follows immediately. Recall from Proposition~\ref{prop:Expl:Empty:Structured} that $\Sem\NotAcc^\dag(\psi)\neq\emptyset$. From Proposition~\ref{prop:Expl:EmptyInter} it is known that $\Sem\Acc^\star(A)\cap\Sem\NotAcc^\dag(B) = \emptyset$ if $A$ is not conflict-relevant for $B$. Hence, if for each $A\in\All\Args(\phi)$ and each $B\in\All\Args(\psi)$ $A$ is not conflict-relevant for $B$ then, by the definitions of the basic explanations for formulas, $\Sem\Acc^\star(\phi)\cap\Sem\NotAcc^\dag(\psi) = \emptyset$. 
\end{proof}

Recall that for Proposition~\ref{prop:Expl:Foil:ApplyContrastive} as well as for Corollary~\ref{cor:Expl:Foil:Alternative} it was assumed that a foil exists. We assume this for the formula counterparts of these results as well. 

\begin{proposition}
  \label{prop:Expl:Foil:ApplyContrastive:Structured}
  Let $\calAF(\AT) = \tuple{\Args,\attack}$ be an AF based on $\AT$, let $\phi\in\calL$ be such that $\Foil(\phi)\neq\emptyset$ and $\Sem\in\{\Adm,\allowbreak \Cmp,\allowbreak \Grd,\allowbreak \Prf,\allowbreak \Sstb\}$. Then: a contrastive acceptance [resp.\ non-acceptance] explanation can be requested for $\phi$, when $\phi$ is at least credulously accepted [resp.\ not skeptically accepted] and for all $\psi\in\Foil(\phi)$, $\psi$ is at least not skeptically accepted [resp.\ credulously accepted]. 
\end{proposition}

\begin{proof}
  Let $\calAF(\AT) = \tuple{\Args,\attack}$ be an AF based on $\AT$, let $\phi\in\calL$ be such that $\Foil(\phi)\neq\emptyset$ and $\Sem\in\{\Adm,\allowbreak \Cmp,\allowbreak \Grd,\allowbreak \Prf,\allowbreak \Sstb\}$. To show that:
      \begin{enumerate}
          \item for each $\ext\in\Sem(\calAF(\AT))$ it is never the case that $\{\phi,\psi\}\subseteq\concs(\ext)$: since $\psi= -\phi$, every argument for $\psi$ rebuts every arguments for $\phi$ and vice versa. Let $A\in\All\Args(\psi)$ and $B\in\All\Args(\phi)$, then $(A,B),(B,A)\in\attack$. Since $\ext\in\Sem(\calAF(\AT))$ is always conflict-free, it follows that $\{\phi,\allowbreak\psi\}\not\subseteq\concs(\ext)$.  
          \item either an argument for $\phi$ is conflict-relevant for an argument for $\psi$ or an argument for $\psi$ is conflict-relevant for an argument for $\phi$: this follows since $(A,B)\in\attack$ and $(B,A)\in\attack$ for any $A\in\All\Args(\psi)$ and $B\in\All\Args(\phi)$.\qedhere
        \end{enumerate} 
\end{proof}
This shows that for ASPIC$^+$, when the foil is defined by the negation of the fact and the first condition for the applicability is fulfilled, the other conditions are fulfilled as well.
From Propositions~\ref{prop:Expl:EmptyInter:Structured} and~\ref{prop:Expl:Foil:ApplyContrastive:Structured} we obtain the following corollary. 

\begin{corollary}
  \label{cor:Expl:Foil:Alternative:Structured}
  Let $\calAF = \tuple{\Args,\attack}$ be an AF based on $\AT$, let $\phi\in\calL$ be such that $\Foil(\phi)\neq\emptyset$ and $\Sem\in\{\Adm,\Cmp,\Grd,\Prf,\Sstb\}$. 
  \begin{itemize}
    \item It is never the case that $\Sem\Cont^\star(\phi,\Foil(\phi))$ is a pair, i.e., is of the form $\Big\langle\Sem\Acc^\star(\phi),\allowbreak \bigcup_{\psi\in \Foil(\phi)}\Sem\NotAcc^\dag(\psi)\Big\rangle$.
    \item $\Sem\ContN^\star(\phi,\Foil(\phi))=\tuple{\Sem\NotAcc^\star(\phi),\allowbreak \bigcup_{\psi\in \Foil(\phi)}\Sem\Acc^\dag(\psi)}$ iff $\Sem\Acc^\dag(\psi) = \emptyset$ for all $\psi\in \Foil(\phi)$.
  \end{itemize}
\end{corollary}

Like in the case of arguments, when the foil is determined by the negation of the fact, the explanation $\Cont$ is never a pair and the explanation $\ContN$ is only a pair when all arguments for the formulas of the foil are not attacked at all.

\section{Conclusion and Discussion}
\label{sec:Conclusion}

The objective of this paper was to provide a strong formal basis for argument-based contrastive explanations. To this end we have employed a basic framework for explanations introduced in~\cite{BorgB20basic} that can be applied on top of an AF as introduced in~\cite{Dung95} and which allows for a variety of explanations of both accepted and non-accepted arguments. We defined contrastive explanations as one explanation highlighting the common elements for the acceptance [resp.\ non-acceptance] of the fact and the non-acceptance [resp.\ acceptance] of the foil. In Section~\ref{sec:Structured} we have refined these explanations such that these can also be applied to an ASPIC$^+$-setting~\cite{Prakken10}. Due to the generality of our approach, a user can fill in a variety of requirements (e.g., on the acceptability and (in)compatibility of fact and foil,  or which approach to structured argumentation is chosen). The result is a general approach to derive contrastive explanations from AFs generated from an abstract or structured setting. Throughout the paper, we have illustrated the usefulness of our explanations in a real-life application with an example for an argumentation-based system employed by the Netherlands Police. To the best of our knowledge this is the first investigation into contrastive local explanations for conclusions derived from both abstract and structured argumentation.

One of the challenges in the literature on contrastive explanations for conclusions derived by some AI system is that the foil is not always explicit. Since an AF comes with a clear notion of conflict (i.e., the attack relation between the arguments of the framework), we were able to introduce a way to find the foil of an argument or formula when it is not provided explicitly. This gives the use of argumentation to explain decisions an advantage over other approaches to AI that do not have such a notion of conflict. 

It is important to note that, for all XAI approaches to contrastive explanations, the choice of the foil is crucial for the final content of the explanation. For example, in the abstract setting, when the AF is a sequence of arguments and counterarguments, the fact is the argument that attacks no other argument and the foil is its direct attacker, then the contrastive explanation does not differ from the basic explanation. However, when the foil is chosen well, the contrastive explanations can highlight the difference between fact and a specific possible scenario, case or exception, as we have shown in our running example. 

For this paper we chose to highlight a particular form of contrast in the context of formal argumentation, i.e., the difference between an accepted [resp.\ non-accepted] argument and a (set of) non-accepted [resp.\ accepted] argument(s). One could interpret our approach as modeling the alternative contrastive explanations from~\cite{Miller18}. There are however other forms of contrast that could be highlighted in a contrastive explanation as well. For example, in terms of accepted arguments, one could compare two accepted but different arguments and return as the explanation the difference in their acceptance explanation; or explain why an argument is $\cup$-accepted but not $\cap$-accepted. 

\medskip

{\bf A note on the computational complexity.} The main complexity bottleneck of our approach is the assumption that the extensions and, by extension, the (non-)acceptance of the considered argument/formula are known. See e.g.,~\cite{DvorakD2018Handbook} for an overview of the complexity of these tasks. The supplementary material of~\cite{BorgB20basic} provides a naive, polynomial time, depth-first-search algorithm for the computation of the basic explanations that requires that the extensions and acceptance status are known. The same algorithm could be applied here as well, since the contrastive explanations are constructed from the basic explanations. 

\medskip

In future work we plan to further investigate different forms of contrastiveness and model these in the context of explanations for argumentation-based conclusions as well as a further study into the choice of the foil and its effect on the explanations. We will also integrate further findings from the social sciences on how humans request, generate, select, interpret and evaluate explanations~\cite{Miller19,SamekWM17explainable,SamekMVHM19towardsXAI}. Moreover, we will study how research on argumentation that aims at modeling other aspects of human reasoning can be applied in our study of explanations (e.g., the use of preferences~\cite{BeirlaenHPS18} or the instantiation with a specific logic, see e.g.,~\cite{BesnardH18handbook}). Finally, we plan to do a user study, to assess the relevance, appropriateness, significance of our proposed explanations.

\section*{Acknowledgements}
This research has been partly funded by the Dutch Ministry of Justice and the Netherlands Police.

\bibliographystyle{plain} 
\bibliography{literature}


\end{document}